\documentclass[runningheads,a4paper]{llncs}

\usepackage{graphicx}
\usepackage{subfigure}
\usepackage{cite}
\usepackage[ruled,vlined]{algorithm2e}
\usepackage{hyperref}
\usepackage{url}

\usepackage{amssymb, amsmath, mathabx, complexity,bm}
\usepackage{epstopdf,cprotect}
\usepackage{tikz,pgfplots}
\graphicspath{{figs/}}
\DeclareMathOperator*{\argmax}{\arg\!\max}
\def\R{\mathbb{R}}
\def\X{\mathcal{X}}
\def\O{\mathcal{O}}
\def\N{\mathcal{N}}
\def\abs#1{\left\lvert#1\right\rvert}
\def\norm#1{\left\lVert#1\right\rVert}
\def\wh#1{\widehat{#1}}
\def\mat#1{\mathbf{#1}}
\def\abovestrut#1{\rule[0in]{0in}{#1}\ignorespaces}
\def\abovespace{\abovestrut{0.20in}}

\begin{document} 

\mainmatter
\titlerunning{Gaussian Process Upper Confidence Bound with Pure Exploration}
\title{Parallel Gaussian Process Optimization with Upper Confidence Bound and Pure Exploration}

\author{Emile Contal\and David Buffoni\and Alexandre Robicquet\and Nicolas Vayatis}
\authorrunning{Contal et al.}

\institute{CMLA, ENS Cachan, CNRS, 61 Avenue du Pr\'{e}sident Wilson, F-94230 Cachan,\\
\{contal, buffoni, vayatis\}@cmla.ens-cachan.fr,\\
alexandre.robicquet@ens-cachan.fr}

\toctitle{Parallel Gaussian Process Optimization with Upper Confidence Bound and Pure Exploration}
\tocauthor{Emile Contal}
\maketitle

\begin{abstract}
In this paper, we consider the challenge of maximizing an unknown function $f$
for which evaluations are noisy and are acquired with high cost.
An iterative procedure uses the previous measures to actively select
the next estimation of $f$ which is predicted to be the most useful.
We focus on the case where the function can be evaluated in parallel
with batches of fixed size 
and analyze the benefit compared to the purely sequential procedure
in terms of cumulative regret.
We introduce the Gaussian Process Upper Confidence Bound and Pure Exploration algorithm
(\textsf{GP-UCB-PE}) which combines the \textsf{UCB} strategy and Pure Exploration in the same batch of evaluations along the parallel iterations.
We prove theoretical upper bounds on the regret with batches of size $K$ for this procedure which show the improvement of the order of $\sqrt{K}$ for fixed iteration cost over  purely sequential versions. Moreover, the multiplicative constants involved have the property of being dimension-free.
We also confirm empirically the efficiency of \textsf{GP-UCB-PE}
on real and synthetic problems compared to state-of-the-art competitors.
\end{abstract} 

\section{Introduction}
\label{intro}
Finding the maximum of a non-convex function by means of sequential noisy observations
is a common task in numerous real world applications. The context of a high dimensional input space  with expensive evaluation cost offers new challenges in order to come up with efficient and valid procedures.
This problem of sequential global optimization arises for example in industrial system design and monitoring 
to choose the location of a sensor to find out the maximum response,
or when determining the parameters of a heavy numerical code designed to maximize the output.
The standard objective in this setting is to minimize the cumulative regret $R_T$,
defined as the sum $\sum_{t=1}^{T} \big( f(x^\star) - f(x_t) \big)$
of the differences between the values of $f$ at the points queried $x_t$
and the true optimum of $f$ noted $x^\star$.
For a fixed horizon $T$, we refer to \cite{hennig2012}.
In the context where the horizon $T$ is unknown,
the query selection has to deal with the exploration/exploitation tradeoff.
Successful algorithms have been developed in
different settings to address this problem such as
experimental design \cite{fedorov1972},
Bayesian optimization \cite{chen2012,guestrin2005,grunewalder2010,srinivas2012,mockus1989,mes2011},
active learning \cite{carpentier2011,chen2013},
multiarmed bandit \cite{auer2002,auer2007,coquelin2007,kleinberg2005,kocsis2006,audibert2011,sutton1998}
and in particular Hierarchical Optimistic Optimization algorithm, \textsf{HOO} \cite{bubeck2008}
for bandits in a generic space, namely $\X$-Armed bandits.
In some cases, it is possible to evaluate the function in parallel with batches of $K$ queries with no increase in cost.
This is typically the case in the sensors location problem if $K$ sensors are available at each iteration,
or in the numerical optimization problem on a cluster of $K$ cores. Parallel strategies have been developed recently in \cite{desautels2012,azimi2010}.
In the present paper, we propose to explore further
the potential of parallel strategies for noisy function optimization with unknown horizon aiming simultaneously at practical efficiency and plausible theoretical results. 
We introduce a novel algorithm called \textsf{GP-UCB-PE} based on the Gaussian process approach
which combines the benefits of the \textsf{UCB} policy with Pure Exploration queries in the same batch of $K$ evaluations of $f$. 
The Pure Exploration component helps to reduce the uncertainty about $f$ in order to support the \textsf{UCB} policy in finding the location of the maximum,
and therefore in increasing the decay of the regret $R_t$ at every timestep $t$.
 In comparison to other algorithms based on Gaussian processes and UCB such as
 \textsf{GP-BUCB} \cite{desautels2012}, the new algorithm discards the need for the initialization phase and offers a tighter control on the uncertainty parameter which monitors overconfidence. As a result, the derived regret bounds do not suffer from the curse of dimensionality since the multiplicative constants obtained are dimension free in contrast with the doubly exponential dependence observed  in previous work.
We also mention that Monte-Carlo simulations can be proposed as an alternative
and this idea has been implemented in the \textit{Simulation Matching} with \textsf{UCB} policy (\textsf{SM-UCB}) algorithm \cite{azimi2010} which we also consider for comparison in the present paper.
Unlike \textsf{GP-BUCB}, no theoretical guarantees for the \textsf{SM-UCB} algorithm are known for the bounds on the number of iterations needed to get close enough to the maximum, therefore the discussion will be reduced to   empirical comparisons over several benchmark problems.
The remainder of the paper is organized as follows. We state the background and our notations in Section \ref{sec:bg}.
We formalize the Gaussian Process assumptions on $f$,
and give the definition of  regret in the parallel setting.
We then describe the GP-UCB-PE algorithm and the main concepts in Section \ref{sec:algo}.
We provide theoretical guarantees through upper bounds
for the cumulative regret of \textsf{GP-UCB-PE} in Section \ref{sec:regret}.
We finally show comparisons of our method and the related algorithms through a series of numerical experiments on real and synthetic functions
in Section \ref{sec:expes}.
\footnote{The documented source codes and the assessment data sets are available online at
  \url{http://econtal.perso.math.cnrs.fr/software/}}

\section{Problem Statement and Background}
\label{sec:bg}
\subsection{Sequential Batch Optimization}
We address the problem of finding in the lowest possible number of iterations
the maximum of an unknown function $f : \X \to \R$ where $\X \subset \R^d$,
denoted by :
\[f(x^\star) = \max_{x \in \X} f(x)~.\]
The arbitrary choice of formulating the optimization problem as a maximization
is without loss of generality,
as we can obviously take the opposite of $f$ if the problem is a minimization one.
At each iteration $t$, we choose a batch of $K$ points in $\X$ called the queries $\{x_t^k\}_{0\leq k<K}$,
and then observe simultaneously the noisy values taken by $f$ at these points,
\[y_t^k = f(x_t^k) + \epsilon_t^k~,\]
where the $\epsilon_t^k$ are independent Gaussian noise $\N(0,\sigma^2)$.

\subsection{Objective}
Assuming that the horizon $T$ is unknown, a strategy has to be good at any iteration.
We denote by $r_t^{(k)}$ the difference between the optimum of $f$ and the point queried $x_t^k$,
\[r_t^{(k)} = f(x^\star) - f(x_t^k)~.\]
We aim to minimize the batch cumulative regret,
\[R_T^K = \sum_{t<T} r_t^K~,\]
which is the standard objective with these formulations of the problem \cite{bubeck2012}.
We focus on the case where the cost for a batch of evaluations of $f$ is fixed.
The loss $r_t^K$ incurred at iteration $t$ is then the simple regret for the batch \cite{bubeck2009},
defined as
\[r_t^K = \min_{k<K} r_t^{(k)}~.\]
An upper bound on $R_T^K$ gives an upper bound of $\frac{R_T^K}{T}$ on the minimum gap between the best point
found so far and the true maximum.
We also provide bounds on the full cumulative regret,
\[R_{TK} = \sum_{t<T} \sum_{k<K} r_t^{(k)}~,\]
which model the case where all the queries in a batch should have a low regret.

\subsection{Gaussian Processes}
In order to analyze the efficiency of a strategy, we have to make some assumptions on $f$.
We want extreme variations of the function to have low probability.

Modeling $f$ as a sample of a Gaussian Process (GP) is a natural way
to formalize the intuition that nearby location are highly correlated.
It can be seen as a continuous extension of multidimensional Gaussian distributions.
We say that a random process $f$ is Gaussian
with mean function $m$ and non-negative definite covariance function (kernel) $k$ written :
\begin{align*}
f &\sim GP(m, k)~,\\
\text{where } m &: \X \to \R\\
\text{and } k &: \X \times \X \to \R^+~,
\end{align*}
when for any finite subset of locations
the values of the random function form a multivariate Gaussian random variable
of mean vector $\bm{\mu}$ and covariance matrix $\mat{C}$
given by the mean $m$ and the kernel $k$ of the GP.
That is, for all finite $n$ and $x_1, \dots, x_n \in \X$,
\begin{align*}
  (f(x_1),\dots, f(x_n)) &\sim \N(\bm{\mu}, \mat{C})~,\\
  \text{with } \bm{\mu}[x_i] &= m(x_i)\\
  \text{and } \mat{C}[x_i,x_j] &= k(x_i, x_j)~.
\end{align*}
If we have the prior knowledge that $f$ is drawn from a GP with zero mean
\footnote{this is without loss of generality as the kernel $k$ can completely define the GP \cite{rasmussen2005}.}
and known kernel,
we can use Bayesian inference conditioned on the observations after $T$ iterations
to get the closed formulae for computing the posterior \cite{rasmussen2005},
which is a GP of mean and variance given at each location $x \in \X$ by :
\begin{align}
  \label{eq:mu}
  \wh{\mu}_{T+1}(x) &= \mat{k}_T(x)^\top \mat{C}_T^{-1}\mat{Y}_T\\
  \label{eq:sigma}
  \text{and } \wh{\sigma}^2_{T+1}(x) &= k(x,x) - \mat{k}_T(x)^\top \mat{C}_T^{-1} \mat{k}_T(x)~,
\end{align}
$\mat{X}_T=\{x_t^k\}_{t<T,k<K}$ is the set of queried locations,
$\mat{Y}_T=[y_t^k]_{x_t^k \in \mat{X}_T}$ is the vector of noisy observations,
$\mat{k}_T(x) = [k(x_t^k, x)]_{x_t^k \in \mat{X}_T}$ is the vector of covariances
between $x$ and the queried points,
and $\mat{C}_T = \mat{K}_T + \sigma^2 \mat{I}$
with $\mat{K}_T=[k(x,x')]_{x,x' \in \mat{X}_T}$ the kernel matrix
and $\mat{I}$ stands for the identity matrix.

\begin{figure}[t]
  \begingroup
  \tikzset{every picture/.style={scale=.75}}
  \begin{center}
%
%
%
%

\definecolor{mycolor1}{rgb}{0.875,0.875,0.875}
\definecolor{mycolor2}{rgb}{0,0.498039215803146,0}

\begin{tikzpicture}

\begin{axis}[%
scale only axis,
width=4in,
height=2in,
yticklabel style={align=right,inner sep=0pt,xshift=-0.1cm},
xtick={-1,-0.5,0,.5,1},
xmin=-1.2, xmax=1.3,
ymin=-2.5, ymax=0.5]

\addplot [fill=mycolor1,draw=black,forget plot] coordinates{ (-1.2,-0.421410583285268)(-1.175,-0.489420505219225)(-1.15,-0.557224102465456)(-1.125,-0.624338522885627)(-1.1,-0.690101981541974)(-1.075,-0.75357799300863)(-1.05,-0.813397345842684)(-1.025,-0.867505859255869)(-1,-0.912813407871538)(-0.975,-0.944911731713062)(-0.95,-0.958497341145127)(-0.925,-0.949684608833108)(-0.9,-0.920903705473503)(-0.875,-0.878287930611012)(-0.85,-0.827394628307923)(-0.825,-0.772210829467251)(-0.8,-0.715360821662011)(-0.775,-0.658551205834719)(-0.75,-0.602915363426654)(-0.725,-0.549230769259095)(-0.7,-0.498048609130932)(-0.675,-0.449770737538613)(-0.65,-0.40469603357818)(-0.625,-0.363048854006163)(-0.6,-0.324996741980193)(-0.575,-0.290661433508235)(-0.55,-0.260125451692161)(-0.525,-0.233435565787386)(-0.5,-0.210603769729149)(-0.475,-0.191606006335368)(-0.45,-0.176378507233417)(-0.425,-0.164811242772483)(-0.4,-0.156737487673444)(-0.375,-0.151917785583111)(-0.35,-0.150015462524581)(-0.325,-0.150559050441666)(-0.3,-0.152884300227538)(-0.275,-0.156045120648796)(-0.25,-0.158681300383645)(-0.225,-0.158841782986855)(-0.2,-0.153816299613864)(-0.175,-0.140175590455579)(-0.15,-0.114408405881147)(-0.125,-0.0753867361184067)(-0.0999999999999999,-0.0267098982853913)(-0.075,0.0271145690098277)(-0.05,0.0822901551241473)(-0.0249999999999999,0.13600790765228)(0,0.186287477869973)(0.0250000000000001,0.231741056773152)(0.05,0.271385974326665)(0.075,0.304520344496309)(0.1,0.330647698796977)(0.125,0.349436218733408)(0.15,0.360703704218678)(0.175,0.364425082760307)(0.2,0.360764607105034)(0.225,0.350140626919796)(0.25,0.333337119495961)(0.275,0.31167951434838)(0.3,0.287276096913147)(0.325,0.263241467000039)(0.35,0.243597535616658)(0.375,0.23221699235153)(0.4,0.229406928798976)(0.425,0.232227152174706)(0.45,0.237404850733996)(0.475,0.242277010978374)(0.5,0.244961671758685)(0.525,0.244214452050528)(0.55,0.239240412914564)(0.575,0.229546194442985)(0.6,0.214839314121564)(0.625,0.194962737442596)(0.65,0.16985278532584)(0.675,0.139511915648822)(0.7,0.103990956729983)(0.725,0.0633774321315392)(0.75,0.0177879096775577)(0.775,-0.0326368971275148)(0.8,-0.0877350541335101)(0.825,-0.147324821416566)(0.85,-0.21120480597817)(0.875,-0.279152935014824)(0.9,-0.350924180914071)(0.925,-0.426246797525892)(0.95,-0.50481659440329)(0.975,-0.586288418864144)(1,-0.670263427273132)(1.025,-0.756269711441706)(1.05,-0.843732040517426)(1.075,-0.93192321388079)(1.1,-1.01988363597341)(1.125,-1.10628562565505)(1.15,-1.18920462610989)(1.175,-1.26575372952912)(1.2,-1.331609199136)(1.225,-1.3808014955446)(1.25,-1.40684813927193)(1.275,-1.40749085376848)(1.3,-1.38717494400226)(1.3,-2.05853188635693)(1.275,-2.01011004758792)(1.25,-1.97341166432438)(1.225,-1.95221745546777)(1.2,-1.9445615834925)(1.175,-1.94479828321212)(1.15,-1.94775790748682)(1.125,-1.94987542633094)(1.1,-1.94898455766073)(1.075,-1.94384627624458)(1.05,-1.93378621651565)(1.025,-1.91846866019681)(1,-1.89776353108548)(0.975,-1.87166857089672)(0.95,-1.84026321641754)(0.925,-1.80368079634864)(0.9,-1.76209155062029)(0.875,-1.71569223965717)(0.85,-1.66469991644623)(0.825,-1.60934844963365)(0.8,-1.54988697421858)(0.775,-1.48657980315863)(0.75,-1.41970756598003)(0.725,-1.34956951203197)(0.7,-1.27648706756824)(0.675,-1.20080890105872)(0.65,-1.12291796651691)(0.625,-1.04324131010502)(0.6,-0.962263918261811)(0.575,-0.880548680926844)(0.55,-0.798765837018946)(0.525,-0.717737330922496)(0.5,-0.638504549366322)(0.475,-0.562431380930542)(0.45,-0.491354494262206)(0.425,-0.427773955627773)(0.4,-0.374999252217767)(0.375,-0.336980668085546)(0.35,-0.31735257963463)(0.325,-0.31591590792323)(0.3,-0.328155740865745)(0.275,-0.349392111280847)(0.25,-0.375884578540566)(0.225,-0.40492889561399)(0.2,-0.434631298406748)(0.175,-0.463664927427997)(0.15,-0.491091899971471)(0.125,-0.516249741627541)(0.1,-0.538684636438911)(0.075,-0.558117305045726)(0.05,-0.574433619709755)(0.0250000000000001,-0.587697803114202)(0,-0.598191400843515)(-0.0249999999999999,-0.606486885245545)(-0.05,-0.613570275167062)(-0.075,-0.621026823764192)(-0.0999999999999999,-0.631275386802559)(-0.125,-0.647725693863161)(-0.15,-0.67450014208674)(-0.175,-0.714847163559256)(-0.2,-0.76749925587829)(-0.225,-0.828884956318991)(-0.25,-0.895512162556806)(-0.275,-0.964604497976478)(-0.3,-1.03414105141884)(-0.325,-1.10268751420089)(-0.35,-1.16921894374707)(-0.375,-1.23298695700702)(-0.4,-1.29343012176705)(-0.425,-1.35011540131603)(-0.45,-1.4026999870616)(-0.475,-1.45090619660949)(-0.5,-1.49450479655915)(-0.525,-1.53330389799095)(-0.55,-1.56714170571723)(-0.575,-1.59588212487207)(-0.6,-1.61941271683511)(-0.625,-1.63764487217634)(-0.65,-1.65051642437003)(-0.675,-1.65799735635985)(-0.7,-1.66009987431418)(-0.725,-1.65689513585408)(-0.75,-1.64854067185128)(-0.775,-1.63532565725012)(-0.8,-1.61774672667973)(-0.825,-1.59663638955399)(-0.85,-1.5733790765429)(-0.875,-1.55025427268798)(-0.9,-1.53088016861361)(-0.925,-1.52041974933496)(-0.95,-1.52458797941153)(-0.975,-1.5456730613971)(-1,-1.58017606241738)(-1.025,-1.62321183099173)(-1.05,-1.67076413868797)(-1.075,-1.7201100634009)(-1.1,-1.76953919962453)(-1.125,-1.81800400939991)(-1.15,-1.86486883484992)(-1.175,-1.90975306176443)(-1.2,-1.95243647767886)};
\addplot [
color=blue,
only marks,
mark=+,
mark options={solid},
forget plot
]
coordinates{
 (-0.9432,-1.2543)(-0.153621012262113,-0.4043)(0.360257825320531,-0.0198877954270356)(1.24252251138805,-1.70980966938825) 
};
\addplot [
color=blue,
solid,
forget plot
]
coordinates{
 (-1.2,-1.18692353048207)(-1.175,-1.19958678349183)(-1.15,-1.21104646865769)(-1.125,-1.22117126614277)(-1.1,-1.22982059058325)(-1.075,-1.23684402820477)(-1.05,-1.24208074226533)(-1.025,-1.2453588451238)(-1,-1.24649473514446)(-0.975,-1.24529239655508)(-0.95,-1.24154266027833)(-0.925,-1.23505217908403)(-0.9,-1.22589193704355)(-0.875,-1.2142711016495)(-0.85,-1.20038685242541)(-0.825,-1.18442360951062)(-0.8,-1.16655377417087)(-0.775,-1.14693843154242)(-0.75,-1.12572801763897)(-0.725,-1.10306295255659)(-0.7,-1.07907424172256)(-0.675,-1.05388404694923)(-0.65,-1.02760622897411)(-0.625,-1.00034686309125)(-0.6,-0.972204729407652)(-0.575,-0.943271779190152)(-0.55,-0.913633578704697)(-0.525,-0.883369731889166)(-0.5,-0.85255428314415)(-0.475,-0.821256101472428)(-0.45,-0.78953924714751)(-0.425,-0.757463322044256)(-0.4,-0.725083804720247)(-0.375,-0.692452371295066)(-0.35,-0.659617203135823)(-0.325,-0.626623282321277)(-0.3,-0.59351267582319)(-0.275,-0.560324809312637)(-0.25,-0.527096731470226)(-0.225,-0.493863369652923)(-0.2,-0.460657777746077)(-0.175,-0.427511377007418)(-0.15,-0.394454273983943)(-0.125,-0.361556214990784)(-0.0999999999999999,-0.328992642543975)(-0.075,-0.296956127377182)(-0.05,-0.265640060021457)(-0.0249999999999999,-0.235239488796632)(0,-0.205951961486771)(0.0250000000000001,-0.177978373170525)(0.05,-0.151523822691545)(0.075,-0.126798480274709)(0.1,-0.104018468820967)(0.125,-0.0834067614470666)(0.15,-0.0651940978763965)(0.175,-0.0496199223338452)(0.2,-0.0369333456508569)(0.225,-0.0273941343470971)(0.25,-0.0212737295223022)(0.275,-0.0188562984662335)(0.3,-0.0204398219762993)(0.325,-0.0263372204615955)(0.35,-0.0368775220089859)(0.375,-0.0523818378670076)(0.4,-0.0727961617093953)(0.425,-0.0977734017265335)(0.45,-0.126974821764105)(0.475,-0.160077184976084)(0.5,-0.196771438803819)(0.525,-0.236761439435984)(0.55,-0.279762712052191)(0.575,-0.32550124324193)(0.6,-0.373712302070124)(0.625,-0.424139286331211)(0.65,-0.476532590595536)(0.675,-0.53064849270495)(0.7,-0.586248055419126)(0.725,-0.643096039950214)(0.75,-0.700959828151234)(0.775,-0.759608350143071)(0.8,-0.818811014176047)(0.825,-0.878336635525107)(0.85,-0.937952361212197)(0.875,-0.997422587335996)(0.9,-1.05650786576718)(0.925,-1.11496379693726)(0.95,-1.17253990541042)(0.975,-1.22897849488043)(1,-1.28401347917931)(1.025,-1.33736918581926)(1.05,-1.38875912851654)(1.075,-1.43788474506268)(1.1,-1.48443409681707)(1.125,-1.528080525993)(1.15,-1.56848126679836)(1.175,-1.60527600637062)(1.2,-1.63808539131425)(1.225,-1.66650947550619)(1.25,-1.69012990179816)(1.275,-1.7088004506782)(1.3,-1.72285341517959) 
};

\end{axis}
\end{tikzpicture}%
    \caption{Gaussian Process inference of the posterior mean $\wh{\mu}$ (blue line) and deviation $\wh{\sigma}$
      based on four realizations (blue crosses).
      The high confidence region (area in grey) is delimited by $\wh{f}^+$ and $\wh{f}^-$.}
    \label{fig:gp}
  \end{center}
  \endgroup
\end{figure}
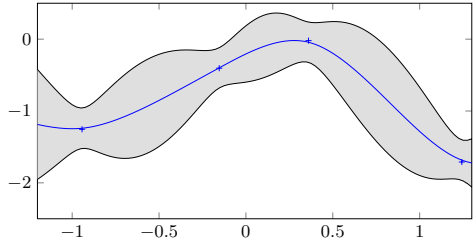 

The three most common kernel functions are:
\begin{itemize}
\item the polynomial kernels of degree $\alpha \in \mathbb{N}$,
$k(x_1,x_2) = (x_1^\top x_2 + c)^\alpha~, ~ c \in \R$,
\item the (Gaussian) Radial Basis Function kernel (RBF or Squared Exponential) with length-scale $l > 0$,
  $k(x_1,x_2) = \exp\Big(-\frac{\norm{x_1,x_2}^2}{2 l^2}\Big)$,
\item the Mat\'{e}rn kernel, of length-scale $l$ and parameter $\nu$,
  \begin{align}
    \label{eq:matern} 
    k(x_1,x_2) = \frac{2^{1-\nu}}{\Gamma(\nu)} \left(\frac{\sqrt{2\nu}\norm{x_1,x_2}}{l}\right)^\nu K_\nu\Big(\frac{\sqrt{2\nu} \norm{x_1,x_2}}{l}\Big)~,
  \end{align}
  where $K_\nu$ is the modified Bessel function of the second kind and order $\nu$.
 \end{itemize}
The Bayesian inference is represented on Figure \ref{fig:gp} in a sample problem in dimension $1$.
The posteriors are based on four observations of a Gaussian Process.
The vertical height of the grey area is proportional to the posterior deviation at each point.

\section{Parallel Optimization Procedure}
\label{sec:algo}
\subsection{Confidence Region}
A key property from the GP framework is that the posterior distribution at a location $x$
has a normal distribution $\N(\wh{\mu}_T(x), \wh{\sigma}^2_T(x))$.
We can then define a upper confidence bound $\wh{f}^+$ and a lower confidence bound $\wh{f}^-$,
such that $f$ is included in the interval with high probability,
\begin{align}
  \label{eq:fp}
  \wh{f}_T^+(x) &= \wh{\mu}_T(x) + \sqrt{\beta_T} \wh{\sigma}_T(x)\\
  \label{eq:fm}
  \text{and } \wh{f}_T^-(x) &= \wh{\mu}_T(x) - \sqrt{\beta_T} \wh{\sigma}_T(x)~,
\end{align}
with $\beta_T \in \O(\log T)$ defined in Section \ref{sec:regret}.

$\wh{f}^+$ and $\wh{f}^-$ are illustrated on Figure \ref{fig:gp}
respectively by the upper and lower envelope of the grey area.
The region delimited in that way, the high confidence region,
contains the unknown $f$ with high probability.
This statement will be a main element in the theoretical analysis of the algorithm
in Section \ref{sec:regret}.

\subsection{Relevant Region}
We define the relevant region $\mathfrak{R}_t$ being the region
which contains $x^\star$ with high probability.
Let $y_t^\bullet$ be our lower confidence bound on the maximum,
\begin{align*}
  y_t^\bullet &= \wh{f}_t^-(x_t^\bullet) \text{, where } x_t^\bullet = \argmax_{x \in \X} \wh{f}_t^-(x)~.
\end{align*}
$y_t^\bullet$ is represented by the horizontal dotted green line on Figure \ref{fig:algo}.
$\mathfrak{R}_t$ is defined as :
\[\mathfrak{R}_t = \Big\{ x \in \X \mid \wh{f}_t^+(x) \geq y_t^\bullet \Big\}~.\]
$\mathfrak{R}_t$ discard the locations where $x^\star$ does not belong with high probability.
It is represented in green on Figure \ref{fig:algo}.
We refer to \cite{freitas2012} for related work in the special case of deterministic Gaussian Process Bandits.

In the sequel, we will use a modified version of the relevant region
which also contains $\argmax_{x \in \X} \wh{f}_{t+1}^+(x)$ with high probability.
The novel relevant region is formally defined by :
\begin{align}
\label{eq:r}
\mathfrak{R}_t^+ = \Big\{ x \in \X \mid \wh{\mu}_t(x) + 2\sqrt{\beta_{t+1}} \wh{\sigma}_t(x) \geq y_t^\bullet \Big\}~.
\end{align}
Using $\mathfrak{R}_t^+$ instead of $\mathfrak{R}_t$ guarantees that
the queries at iteration $t$ will leave an impact
on the future choices at iteration $t+1$.

\subsection{\textsf{GP-UCB-PE}}
\begin{algorithm}[t]
   \caption{\textsf{GP-UCB-PE}}
   \label{alg:algo}
   \DontPrintSemicolon
   \For{$t=0,\dots,T$}{
     Compute $\wh{\mu}_t$ and $\wh{\sigma}_t$ with Eq.\ref{eq:mu} and Eq.\ref{eq:sigma}\;
     $x_t^0 \gets \argmax_{x \in \X} \wh{f}_t^+(x)$\;
     Compute $\mathfrak{R}_t^+$ with Eq.\ref{eq:r}\;
     \For{$k=1,\dots,K-1$}{
       Compute $\wh{\sigma}_t^{(k)}$ with Eq.\ref{eq:sigma}\;
       $x_t^k \gets \argmax_{x \in \mathfrak{R}_t^+} \wh{\sigma}_t^{(k)}(x)$\;
     }
     Query $\{x_t^k\}_{k<K}$\;
   }
\end{algorithm}

We present here the Gaussian Process Upper Confidence Bound with Pure Exploration algorithm,
\textsf{GP-UCB-PE}, a novel algorithm combining two strategies to determine the queries $\{x_t^k\}_{k<K}$
for batches of size $K$.
The first location is chosen according to the \textsf{GP-UCB} rule,
\begin{align}
\label{eq:argmax_f}
x_t^0 = \argmax_{x \in \X} \wh{f}_t^+(x) ~.
\end{align}
This single rule is enough to tackle the exploration/exploitation tradeoff.
The value of $\beta_t$ balances between exploring uncertain regions
(high posterior variance $\wh{\sigma}_t^2(x)$)
and focusing on the supposed location of the maximum
(high posterior mean $\wh{\mu}_t(x)$).
This policy is illustrated with the point $x^0$ on Figure \ref{fig:algo}.

The $K-1$ remaining locations are selected via Pure Exploration
restricted to the region $\mathfrak{R}_t^+$.
We aim to maximize $I_t(\mat{X}_t^{K-1})$, the information gain about $f$
by the locations $\mat{X}_t^{K-1}=\{x_t^k\}_{1\leq k<K}$ \cite{cover1991}.
Formally, $I_t(\mat{X})$ is the reduction of entropy when knowing the values of the observations $\mat{Y}$
at $\mat{X}$, conditioned on $\mat{X}_t$ the observations we have seen so far,
\begin{equation}
  \label{eq:info_gain}
  I_t(\mat{X}) = H(\mat{Y}) - H(\mat{Y} \mid \mat{X}_t)~.
\end{equation}
Finding the $K-1$ points that maximize $I_t$ for any integer $K$ is known to be \NP-complete \cite{ko1995}.
However, due to the submodularity of $I_t$ \cite{guestrin2005},
it can be efficiently approximated by the greedy procedure which selects the points one by one
and never backtracks.
For a Gaussian distribution,
$H(\N(\mat{\mu}, \mat{C})) = \frac{1}{2} \log \det ( 2 \pi e \mat{C})$.
We thus have
$I_t(\mat{X}) \in \O(\log \det \mat{\Sigma})$,
where $\mat{\Sigma}$ is the covariance matrix of $\mat{X}$.
For GP, the location of the single point that maximizes the information gain
is easily computed by maximizing the posterior variance.
For all $~1\leq k < K$ our greedy strategy selects the following points one by one,
\begin{align}
\label{eq:argmax_s}
x_t^k = \argmax_{x \in \mathfrak{R}_t^+} \wh{\sigma}_t^{(k)}(x)  ~,
\end{align}
where $\wh{\sigma}_t^{(k)}$ is the updated variance after choosing $\{x_t^{k'}\}_{k'<k}$.
We use here the fact that the posterior variance does not depend on the values $y_t^k$ of the observations,
but only on their position $x_t^k$.
One such point is illustrated with $x^1$ on Figure \ref{fig:algo}.
These $K-1$ locations reduce the uncertainty about $f$,
improving the guesses of the \textsf{UCB} procedure by $x_t^0$.
The overall procedure is shown in Algorithm \ref{alg:algo}.

\begin{figure}[t]
  \begin{center}
    \includegraphics{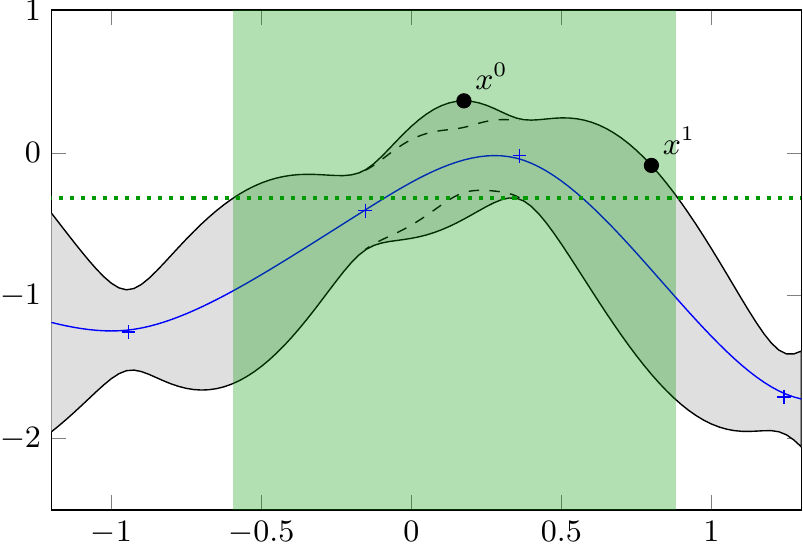}
    \caption{
      Two queries of \textsf{GP-UCB-PE} on the previous example.
      The lower confidence bound on the maximum is represented by the horizontal dotted green line at $y_t^\bullet$.
      The relevant region $\mathfrak{R}$ is shown in light green (without edges).
      The first query $x^0$ is the maximizer of $\wh{f}^+$.
      We show in dashed line the upper and lower bounds with the update of $\wh{\sigma}$
      after having selected $x^0$.
      The second query $x^1$ is the one maximizing the uncertainty inside $\mathfrak{R}^+$,
      an extension of $\mathfrak{R}$ which is not illustrated here.}
    \label{fig:algo}
  \end{center}
\end{figure}

\subsection{Numerical Complexity}
Even if the numerical cost of \textsf{GP-UCB-PE} is insignificant in practice
compared to the cost of the evaluation of $f$,
the complexity of the exact computations of the variances (Eq.\ref{eq:sigma}) is in $\O(n^3)$
and might by prohibitive for large $n=tK$.
One can reduce drastically the computation time
by means of Lazy Variance Calculation \cite{desautels2012},
built on the fact that $\wh{\sigma}_t(x)$ always decreases when $t$ increases for all $x\in\X$.
We further mention that efficient approximated inference algorithms
such as the EP approximation and MCMC sampling \cite{kuss2005}
can be used in order to face the challenge of large $n$.

\section{Regret Bounds}
\label{sec:regret}
\subsection{Main Result}
The main theoretical result of this article is the upper bound on the regret formulated in Theorem \ref{thm:regret}.
We need to adjust the parameter $\beta_t$ such that $f(x)$
is contained by the high confidence region for all iterations $t$
with probability at least $1-\delta$ for a fixed $0<\delta<1$.
\begin{itemize}
\item If $\X$ is finite, then we choose $\beta_t = 2\log(\abs{\X} \frac{\pi_t}{\delta})$
  where $\pi_t>0$ such that $\sum_{t=0}^\infty \pi_t^{-1}=1$.
  We set for example $\beta_t = 2\log\big(\abs{\X} t^2 \frac{\pi^2}{6 \delta}\big)$.

\item If $\X \subset [0,r]^d$ is compact and convex,
  we need the following bounds on the derivatives of $f$,
  \begin{align*}
    \exists a,b>0,\ \forall j\leq d,\ \ 
    \Pr\left(\sup_{x \in \X} \abs{\frac{\partial f}{\partial x_j}} > L\right) \leq a e^{-\frac{L^2}{b^2}}~.
  \end{align*}
  Then, we can set the parameter $\beta_t$ to :
  \[\beta_t = 2 \log\left(t^2 \frac{2 \pi^2}{3 \delta}\right) + 2d\log\left(t^2 d b r \sqrt{\log\Big(\frac{4da}{\delta}\Big)}\right)~.\]
\end{itemize}

The regret bound are expressed in term of $\gamma_{TK}$,
the maximum information gain (Eq. \ref{eq:info_gain}) obtainable by a sequence of $TK$ queries,
\[\gamma_t = \max_{\mat{X} \subset \X, \abs{\mat{X}}=t} I_0(\mat{X})~.\]
Under these assumptions, we obtain the following result.

\begin{theorem}
\label{thm:regret}
Fix $0<\delta<1$ and consider the calibration of $\beta_t$ defined as above,
assuming $f \sim GP(0, k)$ with bounded variance, $\forall x\in\X,~k(x,x)\leq 1$,
then the batch cumulative regret incurred by \textsf{GP-UCB-PE} on $f$
is bounded by $\O\Big(\sqrt{\frac{T}{K}\beta_T \gamma_{TK}}\Big)$ \textit{whp},
More precisely, with $C_1 = \frac{4}{\log(1+\sigma^{-2})}$, and $C_2=\frac{\pi}{\sqrt{6}}$, $\forall T,$
\[\Pr\left(R_T^K \leq \sqrt{C_1 \frac{T}{K} \beta_T \gamma_{TK} + C_2}\right) \geq 1 - \delta~.\]
For the full cumulative regret $R_{TK}$ we obtain similar bounds with $C_1=\frac{36}{\log(1+\sigma^{-2})}$
\[\Pr\left(R_{TK} \leq \sqrt{C_1 T K \beta_T \gamma_{TK} + C_2}\right) \geq 1 - \delta~.\]
\end{theorem}

\subsection{Discussion}
When $K \ll T$, the upper bound for $R_T^K$ is better than the one of sequential \textsf{GP-UCB}
by an order of $\sqrt{K}$,
and equivalent for $R_{TK}$,
when the regrets for all the points in the batch matter.
Compared to \cite{desautels2012}, we remove the need of the initialization phase.
\textsf{GP-UCB-PE} does not need either to multiply the uncertainty parameter $\beta_t$
by $\exp(\gamma_{TK}^\text{init})$ where $\gamma_{TK}^\text{init}$ is equal to
the maximum information gain obtainable by a sequence of $TK$ queries
after the initialization phase.
The improvement can be doubly exponential in the dimension $d$ in the case of RBF Kernels.
To the best of our knowledge,
no regret bounds have been proven for the \textit{Simulation Matching} algorithm.

The values of $\gamma_{TK}$ for different common kernel are reported in Table \ref{tbl:bounds},
where $d$ is the dimension of the space considered
and $\alpha=\frac{d(d+1)}{2\nu +d(d+1)}\leq 1$, $\nu$ being the Mat\'{e}rn parameter.
We also compare on Table \ref{tbl:bounds} the general forms of the bounds for the regret obtained by
\textsf{GP-UCB-PE} and \textsf{GP-BUCB} up to constant terms.
The cumulative regret we obtained with RBF Kernel is of the form
$\tilde{\O}\Big(\sqrt{\frac{T}{K}(\log TK)^d}\Big)$
against $\tilde{\O}\Big(\exp((\frac{2d}{e})^d) \sqrt{\frac{T}{K}(\log TK)^d}\Big)$
for \textsf{GP-BUCB}.

\subsection{Proofs of the Main Result}
In this section, we analyze theoretically the regret bounds for the \textsf{GP-UCB-PE} algorithm.
We provide here the main steps for the proof of 
Theorem \ref{thm:regret}.
On one side the \textsf{UCB} rule of the algorithm provides a regret
bounded by the information we have on $f$ conditioned on the values observed so far.
On the other side, the Pure Exploration part gathers information
and therefore accelerates the decrease in uncertainty.
We refer to \cite{desautels2012} for the proofs of the bounds for \textsf{GP-BUCB}.

For the sake of concision, we introduce the notations
$\sigma_t^k$ for $\wh{\sigma}_t^{(k)}(x_t^k)$
and $\sigma_t^0$ for $\wh{\sigma}_t(x_t^0)$.
We simply bound $r_t^K$ the regret for the batch at iteration $t$ by the simple regret $r_t^{(0)}$
for the single query chosen via the \textsf{UCB} rule.
We then give a bound for $r_t^{(0)}$ which is proportional to the posterior deviations $\sigma_t^0$.
Knowing that the sum of all $(\sigma_t^k)^2$ is not greater than $C_1 \gamma_{TK}$,
we want to prove that the sum of the $(\sigma_t^0)^2$ is less than this bound divided by $K$.
The arguments are based on the fact that the posterior for $f(x)$ is Gaussian,
allowing us to choose $\beta_t$ such that :
\[\forall x \in \X, \forall t<T,\ f(x) \in [\wh{f}_t^-(x), \wh{f}_t^+(x)]\]
holds with high probability.
Here and in the following, ``with high probability'' or \textit{whp}
means ``with probability at least $1-\delta$'' for any $0<\delta<1$,
the definition of $\beta_t$ being dependent of $\delta$.

\begin{table}[t]
\caption{General Forms of Regret Bounds for \textsf{GP-UCB-PE} and \textsf{GP-BUCB}}
\label{tbl:bounds}
\begin{center}
\begin{small}
\begin{tabular}{l*{6}{|c}}
& \multicolumn{3}{c}{$\textsf{GP-UCB-PE}$} & \multicolumn{3}{|c}{$\textsf{GP-BUCB}$}\\
\hline\abovespace
$R_T^K$ & \multicolumn{3}{c}{$\sqrt{\frac{T\log T}{K} \gamma_{TK}}$}
        & \multicolumn{3}{|c}{$C \sqrt{\frac{T \log TK}{K} \gamma_{TK}}$}\\[1ex]
\hline\hline\abovespace
Kernel & \multicolumn{2}{c}{Linear} & \multicolumn{2}{c}{RBF} & \multicolumn{2}{c}{Mat\'{e}rn} \\
\hline\abovespace
$\gamma_{TK}$ & \multicolumn{2}{c}{$d \log{TK}$} & \multicolumn{2}{c}{$(\log TK)^{d+1}$}
              & \multicolumn{2}{c}{$(TK)^\alpha \log TK$} \\[.5ex]
$C$ & \multicolumn{2}{c}{$\exp(\frac{2}{e})$}
                  & \multicolumn{2}{c}{$\exp((\frac{2d}{e})^d)$}
                  & \multicolumn{2}{c}{$e$}\\[.5ex]
\end{tabular}
\end{small}
\end{center}
\end{table}

\begin{lemma}
\label{lem:reg}
For finite $\X$, we have $r_t^K \leq r_t^{(0)} \leq 2 \sqrt{\beta_t} \sigma_t^0$,
and for compact and convex $\X$ following the assumptions of Theorem \ref{thm:regret},
$r_t^K \leq r_t^{(0)} \leq 2 \sqrt{\beta_t} \sigma_t^0 + \frac{1}{t^2}$,
holds with probability at least $1-\delta$.
\end{lemma}
We refer to \cite{srinivas2012} (Lemmas 5.2, 5.8) for the detailed proof of the bound for $r_t^{(0)}$.

Now we show an intermediate result bounding the deviations at the points $x_{t+1}^0$
by the one at the points $x_t^{K-1}$.

\begin{lemma}
\label{lem:s_0_s_k}
The deviation of the point selected by the \textsf{UCB} policy
is bounded by the one for the last point selected by the \textsf{PE} policy
at the previous iteration, \textit{whp},
$\forall t<T,\ \sigma_{t+1}^0 \leq \sigma_{t}^{K-1}$
\end{lemma}
\begin{proof}
By the definitions of $x_{t+1}^0$ (Eq.\ref{eq:argmax_f}),
we have $\wh{f}_{t+1}^+(x_{t+1}^0) \geq \wh{f}_{t+1}^+(x_t^\bullet)$.
Then, we know with high probability that
$\forall x \in \X, \forall t<T,\ \wh{f}_{t+1}^+(x) \geq \wh{f}_t^-(x)$.
We can therefore claim \textit{whp}
$\wh{f}_{t+1}^+(x_{t+1}^0) \geq y_t^\bullet$,
and thus that $x_{t+1}^0 \in \mathfrak{R}_t^+$ \textit{whp}.

We have as a result by the definition of $x_t^{k-1}$ (Eq.\ref{eq:argmax_s})
that $\wh{\sigma}_t^{(k-1)}(x_{t+1}^0) \leq \wh{\sigma}_t^{(k-1)}(x_t^{k-1})$ \textit{whp}.
Using the ``Information never hurts'' principle \cite{krause2005},
we know that the entropy of $f(x)$ for all location $x$
decreases while we observe points $x_t$.
For GP, the entropy is also a non-decreasing function of the variance,
so that :
\begin{align*}
\forall x \in \X,\ \wh{\sigma}_{t+1}^{(0)}(x) \leq \wh{\sigma}_t^{(k-1)}(x) ~.
\end{align*}
We thus prove ${\sigma_{t+1}^0 \leq \sigma_t^{k-1}}$.
\end{proof}

\begin{lemma}
\label{lem:dev_ucb}
The sum of the deviations of the points selected by the \textsf{UCB} policy
are bounded by the one for all the selected points divided by $K$, \textit{whp},
\[\sum_{t=0}^{T-1} \sigma_t^0 \leq \frac{1}{K} \sum_{t=0}^{T-1} \sum_{k=0}^{K-1} \sigma_t^k ~.\]
\end{lemma}
\begin{proof}
Using Lemma \ref{lem:s_0_s_k} and the definitions of the $x_t^k$,
we have that $\sigma_{t+1}^0 \leq \sigma_t^k$ for all $k\geq 1$.
Summing over $k$, we get for all $t\geq 0$,
$\sigma_t^0 + (K-1) \sigma_{t+1}^0 \leq \sum_{k=0}^{K-1} \sigma_t^k$.
Now, summing over $t$ and with $\sigma_0^0\geq 0$ and $\sigma_T^0\geq 0$,
we obtain the desired result.
\end{proof}

Next, we can bound the sum of all posterior variances $(\sigma_t^k)^2$ via the maximum information gain
for a sequence of $TK$ locations.
\begin{lemma}
\label{lem:gamma_tk}
The sum of the variances of the selected points are bounded by a constant factor times $\gamma_{TK}$,
$\exists C_1' \in \R,~ \sum_{t<T} \sum_{k<K} (\sigma_t^k)^2 \leq C_1' \gamma_{TK}$
where $\gamma_{TK}$ is the maximum information gain obtainable by a sequential procedure of length $TK$.
\end{lemma}
\begin{proof}
We know that the information gain for a sequence of $T$ locations $x_t$
can be expressed in terms of the posterior variances $(\wh{\sigma}_{t-1}(x_t))^2$.
The deviations $\sigma_t^k$ being independent of the observations $y_t^k$,
the same equality holds for the posterior variances $(\wh{\sigma}^{(k)}_t(x_t^k))^2$.
See Lemmas 5.3 and 5.4 in \cite{srinivas2012} for the detailed proof,
giving $C_1'=\frac{2}{\log(1+\sigma^{-2})}$.
\end{proof}

\begin{lemma}
\label{lem:end_proof}
The cumulative regret can be bound in terms of the maximum information gain, \textit{whp},
$\exists C_1, C_2 \in \R$,
\[\sum_{t < T} r_t^K \leq \sqrt{\frac{T}{K} C_1 \beta_T \gamma_{TK} + C_2} ~.\]
\end{lemma}
\begin{proof}
Using the previous lemmas and the fact that $\beta_t \leq \beta_T$ for all $t\leq T$,
we have in the case of finite $\X$, \textit{whp},
\begin{align*}
  \sum_{t<T} r_t^K &\leq \sum_{t<T} 2 \sqrt{\beta_t} \sigma_t^0 \text{~, by Lemma \ref{lem:reg}}\\
    &\leq 2 \sqrt{\beta_T} \frac{1}{K} \sum_{t<T} \sum_{k<K} \sigma_t^k \text{~, by Lemma \ref{lem:dev_ucb}}\\
    &\leq 2 \sqrt{\beta_T} \frac{1}{K} \sqrt{ TK \sum_{t<T} \sum_{k<K} (\sigma_t^k)^2} \text{~, by Cauchy-Schwarz}\\
    &\leq 2 \sqrt{\beta_T} \frac{1}{K} \sqrt{TK C'_1 \gamma_{TK}} \text{~, by Lemma \ref{lem:gamma_tk}}\\
    &\leq \sqrt{ \frac{T}{K} C_1 \beta_T \gamma_{TK}} \text{ with } C_1=\frac{4}{\log(1+\sigma^{-2})}~.
\end{align*}
For compact and convex $\X$, a similar reasoning gives :
\[R_T^K \leq \sqrt{\frac{T}{K} C_1 \beta_T \gamma_{TK} + C_2} \text{ with } C_2 = \frac{\pi}{\sqrt{6}} < 2~.\]
\end{proof}

Lemma \ref{lem:end_proof} conclude the proof of Theorem \ref{thm:regret} for the regret $R_T^K$.
The analysis for $R_{TK}$ is simpler,
using the Lemma \ref{lem:regret_ucb} which bounds the regret for the Pure Exploration queries,
leading to $C_1=\frac{36}{\log(1+\sigma^{-2})}$.
\begin{lemma}
\label{lem:regret_ucb}
The regret for the queries $x_t^k$ selected by Pure Exploration in $\mathfrak{R}_t^+$ are
bounded \textit{whp} by, $6 \sqrt{\beta_t} \wh{\sigma}_t(x_t^k)$.
\end{lemma}
\begin{proof}
  As in Lemma \ref{lem:reg}, we have \textit{whp}, for all $t\leq T$ and $k\geq 1$,
  \begin{align*}
    r_t^{(k)} &\leq \wh{\mu}_t(x^\star)+\sqrt{\beta_t}\wh{\sigma}_t(x^\star) - \wh{\mu}_t(x_t^k)+\sqrt{\beta_t}\wh{\sigma}_t(x_t^k)\\
    &\leq \wh{f}_t^-(x^\bullet_t) + 2\sqrt{\beta_t}\wh{\sigma_t}(x^\star) - \wh{\mu}_t(x_t^k)+\sqrt{\beta_t}\wh{\sigma}_t(x_t^k)
     \text{ by definition of } x^\bullet_t\\
    &\leq \wh{\mu}_t(x_t^k) + 2\sqrt{\beta_{t+1}}\wh{\sigma}_t(x_t^k) + 2\sqrt{\beta_t}\wh{\sigma}_t(x_t^k) - \wh{\mu}_t(x_t^k)+\sqrt{\beta_t}\wh{\sigma}_t(x_t^k)
     \text{ by definition of } \mathfrak{R}_t^+\\
    &\leq 3\sqrt{\beta_t}\wh{\sigma}_t(x_t^k) + 2\sqrt{\beta_t}\wh{\sigma}_t(x_t^k) + \sqrt{\beta_t}\wh{\sigma}_t(x_t^k)
     \text{ by definition of } \beta_{t+1}\\
    &\leq 6\sqrt{\beta_t}\wh{\sigma}_t(x_t^k)~.
  \end{align*}
\end{proof}
To conclude the analysis of $R_{TK}$ and prove Theorem \ref{thm:regret},
it suffices to use then the last four steps of Lemma \ref{lem:end_proof}.

\section{Experiments}
\label{sec:expes} 

\begin{figure}[t]
  \begin{center}
    \begin{subfigure}[Himmelblau]{
        \includegraphics[width=2in]{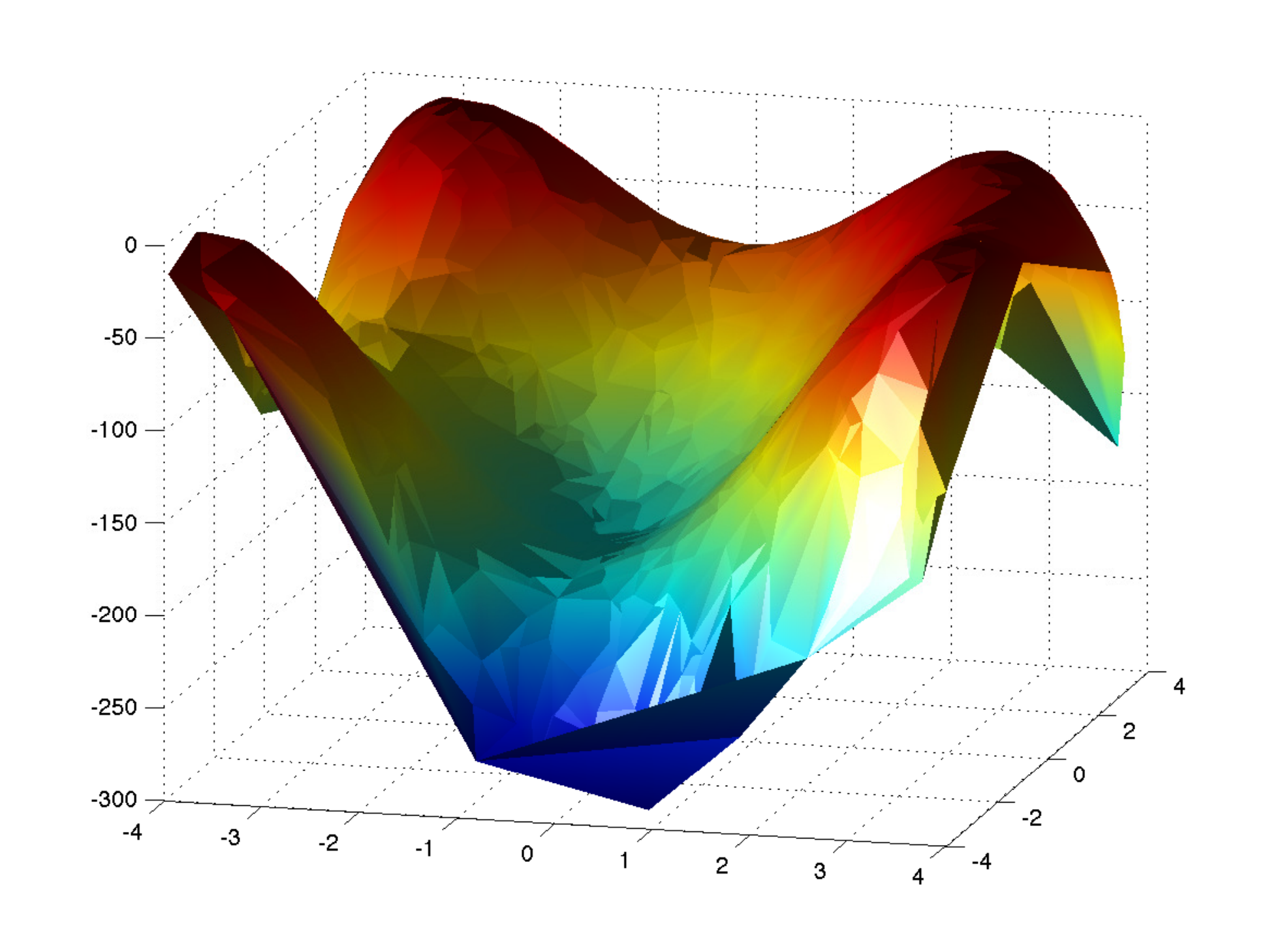}
        \label{fig:himmelblau}}
    \end{subfigure}
    \begin{subfigure}[Gaussian mixture]{
        \includegraphics[width=2in]{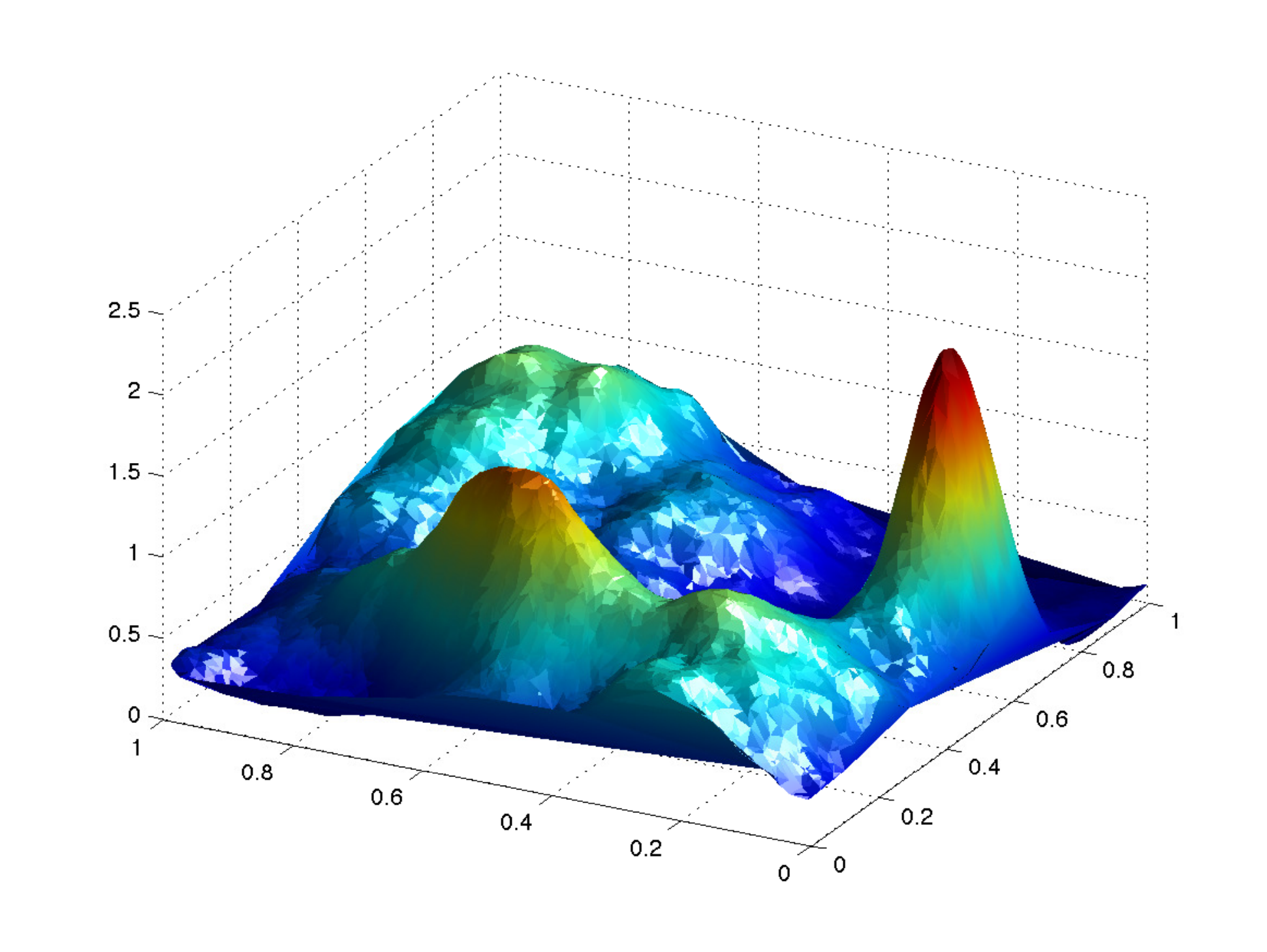}
        \label{fig:gaussian_mix}}
    \end{subfigure}
    \caption{Visualization of the synthetic functions used for assessment}
    \label{fig:data}
  \end{center}
\end{figure} 

\subsection{Protocol}

We compare the empirical performances of our algorithm
against the state of the art of global optimization by batches,
\textsf{GP-BUCB} \cite{desautels2012} and \textsf{SM-UCB} \cite{azimi2010}.
The tasks used for assessment come from
three real applications and two synthetic problems
described here.
The results are shown in Figure \ref{fig:expes}.
For all datasets and algorithms, the size of the batches $K$ was set to $10$
and the learners were initialized with a random subset of $20$ observations $(x_i, y_i)$.
The curves on Figure \ref{fig:expes} show the evolution of the regret $R_t^K$
in term of iteration $t$.
We report the average value with the confidence interval over $64$ experiments.
The parameters for the prior distribution, like the bandwidth of the RBF Kernel,
were chosen by maximization of the marginal likelihood.

\subsection{Description of Data Sets}

\paragraph{Generated GP.}
The \verb/Generated GP/ functions are random GPs drawn from a Mat\'{e}rn kernel (Eq. \ref{eq:matern})
in dimension $2$, with the kernel bandwidth set to $\frac{1}{4}$,
the Mat\'{e}rn parameter $\nu=3$ and noise variance $\sigma^2$ set to $1$.

\paragraph{Gaussian Mixture.}
This synthetic function comes from the addition of three $2$-D Gaussian functions.
at $(0.2, 0.5)$, $(0.9, 0.9)$, and the maximum at $(0.6, 0.1)$.
We then perturb these Gaussian functions with smooth variations
generated from a Gaussian Process with Mat\'{e}rn Kernel and very few noise.
It is shown on Figure \ref{fig:gaussian_mix}.
The highest peak being thin, the sequential search for the maximum of this function
is quite challenging.

\paragraph{Himmelblau Function.}
The \verb/Himmelblau/ task is another synthetic function in dimension $2$.
We compute a slightly tilted version of the Himmelblau's function,
and take the opposite to match the challenge of finding its maximum.
This function presents four peaks but only one global maximum.
It gives a practical way to test the ability of a strategy to manage exploration/exploitation tradeoffs.
It is represented in Figure \ref{fig:himmelblau}.

\paragraph{Mackey-Glass Function.}
The Mackey-Glass delay-differential equation
\footnote{\url{http://www.scholarpedia.org/article/Mackey-Glass_equation}}
is a chaotic system in dimension $6$, but without noise.
It models real feedback systems and is used in physiological domains such as
hematology, cardiology, neurology, and psychiatry.
The highly chaotic behavior of this function makes it an exceptionally difficult optimization problem.
It has been used as a benchmark for example by \cite{flake2002}.

\paragraph{Tsunamis.}
Recent post-tsunami survey data as well as the numerical simulations of \cite{hill2012}
have shown that in some cases the run-up,
which is the maximum vertical extent of wave climbing on a beach,
in areas which were supposed to be protected by small islands in the vicinity of coast,
was significantly higher than in neighboring locations.
Motivated by these observations \cite{stefanakis2012} investigated this phenomenon
by employing numerical simulations using the VOLNA code \cite{dutykh2011}
with the simplified geometry of a conical island sitting on a flat surface in front of a sloping beach.
Their setup was controlled by five physical parameters
and their aim was to find with confidence and with the least number of simulations
the maximum run-up amplification on the beach directly behind the island,
compared with the run-up on a lateral location, not influenced by the presence of the island. 
Since this problem is too complex to treat analytically,
the authors had to solve numerically the Nonlinear Shallow Water Equations.

\begin{figure*}[t]
  \begingroup
  \tikzset{every picture/.style={scale=.62}}
  \centering
  \begin{subfigure}[Generated GP]{
    \includegraphics[width=.28\textwidth]{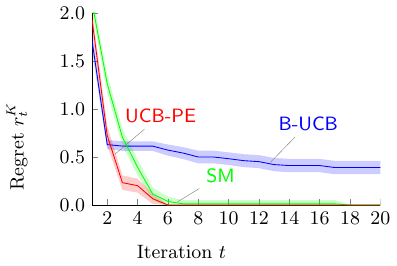}
    \label{fig:expe_gp}}
  \end{subfigure}
  \begin{subfigure}[Himmelblau]{
    \includegraphics[width=.28\textwidth]{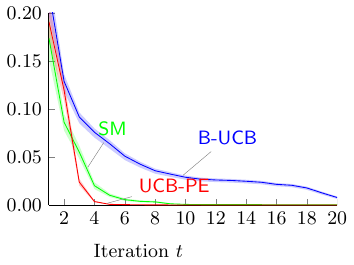}
    \label{fig:expe_himmelblau}}
  \end{subfigure}
  \begin{subfigure}[Gaussian Mixture]{
    \includegraphics[width=.28\textwidth]{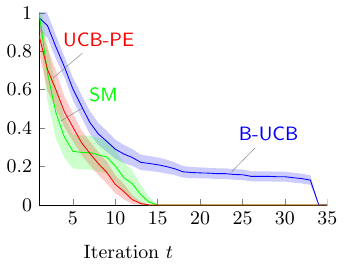}
    \label{fig:expe_gaussian_mix}}
  \end{subfigure}\\
  \begin{subfigure}[Mackey-Glass]{
    \includegraphics[width=.28\textwidth]{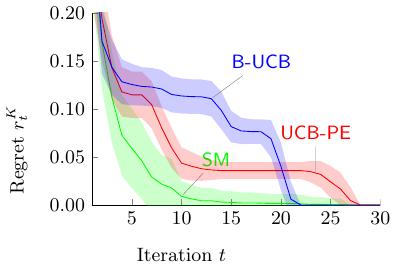}
    \label{fig:expe_mg}}
  \end{subfigure}
  \begin{subfigure}[Tsunamis]{
    \includegraphics[width=.28\textwidth]{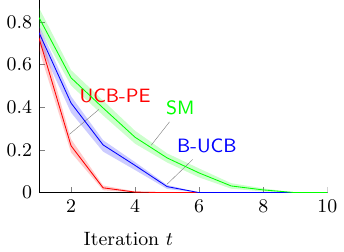}
    \label{fig:expe_tsunamis}}
  \end{subfigure}
  \begin{subfigure}[Abalone]{
    \includegraphics[width=.28\textwidth]{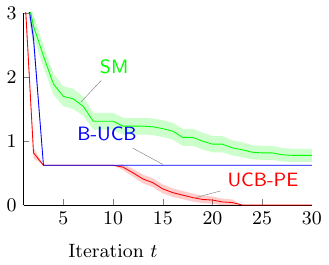}
    \label{fig:expe_abalone}}
  \end{subfigure}\\
  \endgroup
  \caption{Experiments on several real and synthetics tasks.
    The curves show the decay of the mean of the simple regret $r_t^K$ with respect to the iteration $t$,
    over $64$ experiments.
    We show with the translucent area the confidence intervals.
    }
  \label{fig:expes}
\end{figure*}

\paragraph{Abalone.}
The challenge of the \verb/Abalone/ dataset is to predict the age of a specie of sea snails
from physical measurements.
It comes from the study by \cite{nash1994} and it is provided by the UCI Machine Learning Repository.
\footnote{\url{http://archive.ics.uci.edu/ml/datasets/Abalone}}
We use it as a maximization problem in dimension $8$.

\subsection{Comparison of Algorithms}
The algorithm \textsf{SM} ---Simulation Matching--- described in \cite{azimi2010},
with \textsf{UCB} base policy,
has shown similar results to \textsf{GP-UCB-PE} on synthetic functions
(Figures \ref{fig:expe_gp}, \ref{fig:expe_himmelblau}, \ref{fig:expe_gaussian_mix})
and even better results on chaotic problem without noise (Figure \ref{fig:expe_mg}),
but performs worse on real noisy data (Figures \ref{fig:expe_tsunamis}, \ref{fig:expe_abalone}).
On the contrary, the initialization phase of \textsf{GP-BUCB}
leads to good regret on difficult real tasks (Figure \ref{fig:expe_tsunamis}),
but looses time on synthetic Gaussian or polynomial ones
(Figures \ref{fig:expe_gp}, \ref{fig:expe_himmelblau}, \ref{fig:expe_gaussian_mix}).
The number of dimensions of the \verb/Abalone/ task
is already a limitation for \textsf{GP-BUCB} with the RBF kernel,
making the initialization phase time-consuming.
The mean regret for \textsf{GP-BUCB} converges to zero abruptly after the initialization phase at iteration $55$,
and is therefore not visible on Figure \ref{fig:expe_abalone},
as for \ref{fig:expe_gaussian_mix} where its regret decays at iteration $34$.

\textsf{GP-UCB-PE} achieves good performances on both sides.
We obtained better regret on synthetic data
as well as on real problems from the domains of physics and biology.
Moreover, the computation time of \textsf{SM} was two order of magnitude longer than the others.

\section{Conclusion}
We have presented the \textsf{GP-UCB-PE} algorithm
which addresses the problem of finding in few iterations the maximum
of an unknown arbitrary function observed via batches of $K$ noisy evaluations.
We have provided theoretical bounds for the cumulative regret
obtained by \textsf{GP-UCB-PE} in the Gaussian Process settings.
Through parallelization, these bounds improve the ones for the state-of-the-art of sequential GP optimization
by a ratio of $\sqrt{K}$,
and are strictly better than the ones for \textsf{GP-BUCB}, a concurrent algorithm for parallel GP optimization.
We have compared experimentally our method to \textsf{GP-BUCB}
and \textsf{SM-UCB}, another approach for parallel GP optimization lacking of theoretical guarantees.
These empirical results have confirmed the effectiveness of \textsf{GP-UCB-PE}
on several applications.

\paragraph{}
The strategy of combining in the same batch
some queries selected via Pure Exploration is an intuitive idea
that can be applied in many other methods.
We expect for example to obtain similar results with the
Maximum Expected Improvement policy (\textsf{MEI}).
Any proof of regret bound that relies on the fact that the uncertainty decreases with the exploration
should be easily adapted to a paralleled extension with Pure Exploration.

On the other hand, we have observed in practice that the strategies which focus more on exploitation
often lead to faster decrease of the regret,
for example the strategy that uses $K$ times the \textsf{GP-UCB} criterion
with updated variance.
We conjecture that the regret for this strategy is unbounded for general GPs,
justifying the need for the initialization phase of \textsf{GP-BUCB}.
However, it would be relevant to specify formally
the assumptions needed by this greedy strategy to guarantee good performances.

\bibliography{biblio}
\bibliographystyle{splncs}

\end{document}